\documentclass[twoside]{article}
\usepackage[most]{tcolorbox}
\usepackage{mdframed} %
\usepackage{graphicx}
\usepackage{xcolor} 

\usepackage{amsthm}
\usepackage{amsmath}
\usepackage{amsfonts}
\usepackage{float}
\usepackage[colorlinks=true, allcolors=blue]{hyperref}
\usepackage{amsmath,amssymb}
\newcommand{\E}{\mathbb{E}}

\newtheorem{theorem}{Theorem}
\newtheorem{lemma}{Lemma}

\newtheorem{corollary}{Corollary}
\newtheorem{assumption}{Assumption}

\theoremstyle{definition}
\newtheorem{innerdefinition}{Definition}
\newenvironment{definition}
  {\begin{mdframed}[backgroundcolor=gray!10, linewidth=0pt]\begin{innerdefinition}}
  {\end{innerdefinition}\end{mdframed}}

\usepackage[preprint]{aistats2026}
\usepackage[round]{natbib}

\begin{document}

%

%

\twocolumn[

\aistatstitle{Geometric Convergence Analysis of Variational Inference via Bregman Divergences}

\aistatsauthor{ Sushil Bohara \And Amedeo Roberto Esposito
}

\aistatsaddress{MBZUAI, Abu Dhabi, UAE \And  OIST, Okinawa, Japan} ]

\begin{abstract}
 Variational Inference (VI) provides a scalable framework for Bayesian inference by optimizing the Evidence Lower Bound (ELBO), but convergence analysis remains challenging due to the objective's non-convexity and non-smoothness in Euclidean space. We establish a novel theoretical framework for analyzing VI convergence by exploiting the exponential family structure of distributions. We express negative ELBO as a Bregman divergence with respect to the log-partition function, enabling a geometric analysis of the optimization landscape. We show that this Bregman representation admits a weak monotonicity property that, while weaker than convexity, provides sufficient structure for rigorous convergence analysis. By deriving bounds on the objective function along rays in parameter space, we establish properties governed by the spectral characteristics of the Fisher information matrix. Under this geometric framework, we prove non-asymptotic convergence rates for gradient descent algorithms with both constant and diminishing step sizes. 
\end{abstract}

\section{Introduction}
Variational inference (VI) allows us to approximate intractable posterior distributions in probabilistic models by formulating Bayesian inference as an optimization problem \citep{blei2017variational, jordan1999introduction, hoffman2013stochastic}.
The primary goal is to compute the posterior distribution of latent variables $z$ given the observed data $x$. According to Bayes' rule, this posterior is given by:
\[
p(z \mid x) = \frac{p(x, z)}{p(x)}
\]
where $p(x, z)$ is the joint distribution. However, computing this posterior is often intractable in practice. The primary difficulty lies in computation of the marginal likelihood, or evidence, term in the denominator:
\[
p(x) = \int p(x, z)\, dz
\]
This integral requires considering all possible configurations of the latent variables, which is computationally infeasible for high-dimensional or complex models.

Instead of calculating the true posterior $p(z \mid x)$ directly, VI suggests using a simpler distribution $q_{\phi}(z)$ from a parameterized family (for example, a Gaussian). The goal is to find the member of this family that is closest to the true posterior. One way to measure this closeness is through the Kullback-Leibler (KL) divergence, $D_{KL}(q_{\phi}(z) \| p(z\mid x))$ \citep{kullback1951, 
cover2006}. It is well known that minimizing this divergence is the same as maximizing a lower bound on the log-evidence, which is called the Evidence Lower Bound (ELBO) or minimizing negative ELBO \citep{jordan1999introduction,  blei2017variational}.  

In this paper, we work with the exponential family of distributions. Although this family has a rich geometric structure through its natural parameterization, the optimization problem still poses challenges. The negative ELBO objective is usually non-convex in the natural parameters, which complicates theoretical convergence analysis. Thus, we observe a gap between the impressive empirical success of natural gradient methods in VI and theoretical guarantees from the classical convex optimization theory. To bridge this theory-practice gap, we need to better understand the geometric structure that explains the fast convergence despite non-convexity. Our work tackles this need with a new ray-wise geometric analysis that uncovers hidden local structure in the negative ELBO. This analysis provides convergence guarantees for natural gradient methods without relying on global convexity assumptions.

\subsection{Contributions}

Our work makes the following contributions:

\begin{enumerate}
    \item We show that the negative ELBO admits a global lower bound at any point through a monotonicity property that provides useful structure beyond convexity (Theorem~\ref{prop:monotonicity}, Figure~\ref{fig:elbo-geometry}).  
    \item We establish ray-wise spectral bounds that yield two-sided quadratic control of the negative ELBO objective (Theorem~\ref{thm:ray-bounds}, Figure~\ref{fig:ray-geometry}).  
    \item We derive non-asymptotic convergence guarantees for both gradient descent and natural gradient descent under constant and diminishing step sizes, fully based on our geometric framework (Section~\ref{sec:convergence}).  
    \item We introduce an integral representation that links the objective to Fisher-spectrum averages along rays (Theorem~\ref{thm:integral-rep}) and derive one-point PL-type inequalities with ray-dependent constants (Theorem~\ref{thm:one-point}).  
\end{enumerate}

\section{Related Work}

\subsection{Empirical Success vs. Theoretical Understanding}

Variational inference has demonstrated remarkable success in a variety of applications from deep generative models \citep{kingma2014vae} to topic modeling \citep{blei2003} and natural gradient methods have emerged as the method of choice for VI optimization \citep{salimans2013, khan2018}. At the same time, the scalability of these techniques has been further demonstrated by stochastic variational inference \citep{hoffman2013stochastic}, which enables VI in large datasets through mini-batch optimization.

However, the success of these optimization techniques in practice is not well understood theoretically. Although we do have some understanding, conventional VI convergence analysis frequently provides only asymptotic guarantees without finite-time rates \citep{wang2005} or assumes conditions such as global strong convexity \citep{xu2018} which do not typically hold in practice. Information geometry offers some insight into natural gradient methods \citep{amari1998}, but rigorous convergence analysis has mainly focused on particular cases \citep{lin2019} and depend on constants which can't be calculated easily \citep{nemirovsky1983, nesterov2018}. 

This discrepancy between theory and practice is particularly evident in the gradient descent setting. As a result of insufficient theory, effective hyperparameter selection and the development of reliable algorithms have been hampered by the lack of explicit convergence rates.

\subsection{Non-convex Optimization}

The main challenge in understanding VI optimization is the non-convexity of the negative ELBO objective. Recent progress in non-convex optimization theory, particularly with Polyak-\L ojasiewicz (PL) inequalities \citep{polyak1963} and their extensions \citep{karimi2016, necoara2019}, has shown that linear convergence is achievable without global strong convexity. However, current PL analyses usually need uniform constants that might be too conservative for VI applications as the variational inference objective requires more local analysis. Building on this insight, we exploit the geometric features of exponential families and develop a more local analysis to derive the convergence rates. 

Bregman divergence analysis \citep{bregman1967} has been applied to mirror descent \citep{nemirovsky1983,beck2003}, but its use in VI has been restricted to approximation quality rather than convergence rates. Likewise, insights from information geometry \citep{amari2016} have guided algorithm design but have not been transformed into concrete convergence theory. We provide a framework that can take advantage of the specific geometric properties of VI objective as Bregman divergence while offering convergence guarantees.

\section{Preliminaries}
In this paper, we study the variational inference optimization problem of minimizing negative ELBO. Let $\Omega \subset \mathbb{R}^d$ denote the parameter space, assumed to be compact. Then, the objective is

\begin{equation} \label{eq:neg_elbo}
    \min_{\phi \in \Omega} L(\phi), 
    \qquad 
    L(\phi) = \mathbb{E}_{q_\phi}\!\left[\log q_\phi(z) - \log p(x,z)\right],
\end{equation}
where $L(\phi)$ is the negative ELBO and $\phi^* = \arg\min_{\phi \in \Omega} L(\phi)$ is the optimal variational parameter. If $\Omega$ is compact and $L$ is continuous, the minimum is attained, so that $\phi^* \in \arg\min_{\phi \in \Omega} L(\phi)$ is well defined.
Throughout, we focus on exponential family distributions, which provide the natural setting for our geometric analysis.

\begin{definition}[Exponential Family]
Let $(Z,\mathcal{F},\nu)$ be a measurable space.  
The exponential family distribution is
\begin{equation}
    q_\phi(z) = h(z)\exp\!\left(\langle \phi, T(z)\rangle - A(\phi)\right),
    \quad \phi \in \Omega \subset \mathbb{R}^d,
\end{equation}
where $h$ is the base measure, $T$ is the sufficient statistic, and $A$ is the log-partition function.  
The mean parameter and Fisher information matrix are
\[
\mu(\phi) = \nabla A(\phi), 
\qquad 
H(\phi) = \nabla^2 A(\phi).
\]
\end{definition}

We express our optimization objective in terms of the Bregman divergence between two distributions of the exponential family. 

\begin{definition}[Bregman Divergence]
Let $f$ be a strictly convex function.  
The Bregman divergence is defined as
\[
D_f(u \| v) = f(u) - f(v) - \langle \nabla f(v), u - v \rangle.
\]
In particular, the log-partition function $A$ induces
\[
D_A(\phi^* \| \phi) = A(\phi^*) - A(\phi) - \langle \mu(\phi), \phi^* - \phi \rangle.
\]
\end{definition}
We use following assumptions as the theoretical benchmark and do the full analysis without using any global convexity. 
\begin{assumption}[]
The true posterior distribution belongs to  the exponential family. 
$A \in C^2(\Omega)$ is strictly convex; and $H(\phi)\succ 0$ for all $\phi \in \Omega$.
\end{assumption}

\begin{assumption}[]
The joint distribution belongs to the same exponential family:
\[
p(x,z) = h(x,z)\exp\!\left(\langle \theta, T(x,z)\rangle - A(\theta)\right),
\quad \theta \in \Theta,
\]
for some natural parameter $\theta$ and sufficient statistic $T(x,z)$.
\end{assumption}

For convergence analysis, we use Euclidean gradient descent and natural gradient descent. 

\begin{definition}[Gradient Descent (GD)]
Given step sizes $\{\gamma_k\}_{k\geq 0}$, the Euclidean gradient descent update is
\[
\phi_{k+1} = \phi_k - \gamma_k \nabla L(\phi_k),
\]
where $\nabla L(\phi)$ is the Euclidean gradient of the negative ELBO.
\end{definition}

\begin{definition}[Natural Gradient Descent (NGD)]
Given step sizes $\{\eta_k\}_{k\geq 0}$, the natural gradient descent update is
\[
\phi_{k+1} = \phi_k - \eta_k H(\phi_k)^{-1}\nabla L(\phi_k),
\]
where $H(\phi) = \nabla^2 A(\phi)$ is the Fisher information matrix. 
\end{definition}

\section{Geometric Structure of the Negative ELBO}

In this section, we establish the geometric structure of the negative ELBO. Our further analysis will not rely on the convexity; rather, it will utilize the geometric structure arising from its exponential-family parameterization established in this section. We will establish two structures: (i) the objective can be represented as a Bregman divergence (Section \ref{sec:bregman}), and (ii) Negative ELBO satisfies a local monotonicity condition around a reference point (Section \ref{sec:monotone}) that ensures sufficient regularity for convergence analysis

\subsection{Bregman Divergence Representation} \label{sec:bregman}

We begin by establishing that the negative ELBO can be expressed as the Bregman divergence for exponential family of distributions. 

\begin{theorem}[Negative ELBO as a Bregman divergence]\label{thm:bregman}
Under Assumptions~1--2, suppose the joint distribution $p(x,z)$ is exponential family in $z$ with natural parameter $\phi^*$ i.e., 
\(
p(x,z) = h(z)\exp\!\big(\langle\phi^*,T(z)\rangle-A(\phi^*)\big).
\)
Let
\(
L(\phi) \;=\; \mathbb{E}_{q_\phi}\!\left[\log q_\phi(z)-\log p(x,z)\right]
\) 
be the negative ELBO, where $q_\phi(z) = h(z)\exp(\langle\phi,T(z)\rangle-A(\phi))$. Then
\[
L(\phi) = D_A\!\left(\phi^*\middle\Vert\phi\right),
\]
where $D_A$ is the Bregman divergence generated by $A$. Moreover,
\begin{align}
\nabla L(\phi) &= H(\phi)\,(\phi-\phi^*) \\
H(\phi)^{-1}\nabla L(\phi) &= \phi-\phi^*, 
\end{align}
where $H(\phi) = \nabla^2 A(\phi)$ is the Fisher information matrix.
\end{theorem}

\begin{proof}
Substituting the exponential family forms gives
\begin{align}
L(\phi) &= \E_{q_\phi}\!\left[\log q_\phi(z)-\log p(x,z)\right] \nonumber\\
&= \E_{q_\phi}\!\left[(\langle\phi,T(z)\rangle - A(\phi))-(\langle\phi^*,T(z)\rangle - A(\phi^*))\right] \nonumber\\
&= \E_{q_\phi}\!\left[\langle\phi-\phi^*,T(z)\rangle\right] - A(\phi) + A(\phi^*) \nonumber\\
&= \langle \phi-\phi^*,\,\mu(\phi)\rangle - A(\phi) + A(\phi^*), 
\end{align}
Using $\mu(\phi)=\E_{q_\phi}[T(z)]=\nabla A(\phi)$ and rearranging, 
\begin{align}
L(\phi) &= A(\phi^*) - A(\phi) - \langle \mu(\phi), \phi^*-\phi\rangle \nonumber\\
&= D_A(\phi^*\Vert \phi). 
\end{align}

Differentiating $D_A(\phi^*\Vert\phi)$ with respect to $\phi$ yields
\begin{align}
\nabla L(\phi)&=\nabla_\phi\!\left[A(\phi^*)-A(\phi)
-\langle\mu(\phi),\phi^*-\phi\rangle\right] \nonumber\\
&=-\mu(\phi)-H(\phi)(\phi^*-\phi)+\mu(\phi) \nonumber\\
&=H(\phi)(\phi-\phi^*). \qedhere
\end{align}
\end{proof}

This representation connects VI optimization to the well-developed theory of Bregman divergences. Moreover, it also reveals why natural gradient methods are particularly effective, which we prove rigorously in Section \ref{sec:ngd}). However, it can be noticed early that the natural gradient, $H(\phi)^{-1} \nabla L(\phi) = \phi - \phi^*$, points directly towards the optimum compared to the descent of the Euclidean gradient. 

\subsection{Monotonicity around a point} \label{sec:monotone}

While Bregman divergences are not generally convex, they satisfy a weaker but still useful monotonicity property that we establish in this section. To derive this monotonicity we will leverage the following result: 
\begin{lemma}[Bregman Three-Point Identity]\label{lem:three-point}
For any strictly convex function $A$ and points $u, v, w \in \Omega$:
\begin{equation}\label{eq:three-point}
\begin{split}
D_A(u \| v) - D_A(u \| w) 
- \langle \mu(w) - \mu(v), u - w \rangle \\ = D_A(w \| v)
\end{split}
\end{equation}
where $\mu(\cdot) = \nabla A(\cdot)$.
\end{lemma}

\begin{proof}
Writing out each Bregman divergence:
\begin{align*}
D_A(u\|v) &= A(u) - A(v) - \langle \mu(v), u-v\rangle\\
D_A(u\|w) &= A(u) - A(w) - \langle \mu(w), u-w\rangle
\end{align*}

Consider the difference:
\begin{align*}
&D_A(u\|v) - D_A(u\|w) - \langle \mu(w) - \mu(v), u-w\rangle\\
&= [A(u) - A(v) - \langle \mu(v), u-v\rangle]\\
&\quad - [A(u) - A(w) - \langle \mu(w), u-w\rangle]\\
&\quad - \langle \mu(w) - \mu(v), u-w\rangle\\
&= -A(v) - \langle \mu(v), u-v\rangle + A(w)\\
&\quad + \langle \mu(w), u-w\rangle\\
&\quad - \langle \mu(w) - \mu(v), u-w\rangle
\end{align*}

Expanding the last inner product:
\[
\langle \mu(w) - \mu(v), u-w\rangle = \langle \mu(w), u-w\rangle - \langle \mu(v), u-w\rangle
\]

Substituting and noting the $\langle \mu(w), u-w\rangle$ terms cancel:
\begin{align*}
&= -A(v) - \langle \mu(v), u-v\rangle + A(w)\\
&\quad + \langle \mu(v), u-w\rangle\\
&= A(w) - A(v) + \langle \mu(v), u-w - u+v\rangle\\
&= A(w) - A(v) + \langle \mu(v), v - w\rangle\\
&= A(w) - A(v) - \langle \mu(v), w - v\rangle\\
&= D_A(w\|v)
\end{align*}

Therefore:
\[
D_A(u\|v) - D_A(u\|w) - \langle \mu(w) - \mu(v), u-w\rangle = D_A(w\|v)
\]
\end{proof}
This identity enables us to establish the key monotonicity property of the negative ELBO.

\begin{theorem}[Monotonicity around Reference point]\label{prop:monotonicity}
Under Assumptions 1 and 2, the negative ELBO satisfies:
\begin{equation}\label{eq:monotonicity}
L(\phi') \geq L(\phi) + \langle \phi - \phi^*, \mu(\phi') - \mu(\phi) \rangle
\end{equation}
for all $\phi, \phi' \in \Omega$.
\end{theorem}

\begin{proof}
Apply the three-point identity (Lemma \ref{lem:three-point}) with $u = \phi^*$, $v = \phi'$, and $w = \phi$:
\begin{equation}
\begin{split}
D_A(\phi^* \| \phi') - D_A(\phi^* \| \phi) 
- \langle \mu(\phi) - \mu(\phi'), \phi^* - \phi \rangle \\
= D_A(\phi \| \phi')
\end{split}
\end{equation}

Since $L(\cdot) = D_A(\phi^* \| \cdot)$ from Theorem \ref{thm:bregman} and $D_A(\phi \| \phi') \geq 0$:
\begin{equation}
\begin{split}
L(\phi') - L(\phi) \geq \langle \mu(\phi) - \mu(\phi'), \phi^* - \phi \rangle \\
= \langle \phi - \phi^*, \mu(\phi') - \mu(\phi) \rangle
\end{split}
\end{equation}
\end{proof}

This monotonicity property is weaker than convexity but provides global lower bound structure for convergence analysis. It ensures that the objective function has well-behaved directional derivatives along the gradient direction, which will be crucial for establishing convergence rates. Despite global non-convexity, it ensures that local optimization behavior is well-controlled.

These properties set the stage for our ray-wise analysis, which will exploit the local geometric structure to establish precise convergence guarantees without requiring global convexity assumptions. Before that, we can visually illustrate the monotonicity property through an example. 

\paragraph{Example (Bernoulli exponential family).}
To highlight the geometric features we discovered, we explicitly characterize them in a simple and concrete example: the Bernoulli distribution.
Its natural parameter is $\phi \in \mathbb{R}$, with log-partition function
\begin{align*}
A(\phi) &= \log(1+e^\phi),\\
\mu(\phi) &= \nabla A(\phi) = \sigma(\phi) = \frac{1}{1+e^{-\phi}}.
\end{align*}

For a fixed ``true'' parameter $\phi^{*} \in \mathbb{R}$, the negative ELBO
coincides with the Bregman divergence
\[
L(\phi) = D_A(\phi^{*} \,\|\, \phi)
= A(\phi^{*}) - A(\phi) - \mu(\phi)\,(\phi^{*} - \phi).
\]
This function is minimized uniquely at $\phi = \phi^{*}$ with $L(\phi^{\*})=0$. In Figure \ref{fig:elbo-geometry}, we can visualize the global lower bound at any point of this function as derived in Section \ref{sec:monotone}. We set $\phi^{\*}=1$ (minimizer shown in green) and choose
a point $\phi_0=-1$ (shown in red) from which the tangent and monotonicity inequalities are drawn. It can be seen that our lower bound is global. 

\begin{figure}[H]
  \centering
  \includegraphics[width=1\linewidth]{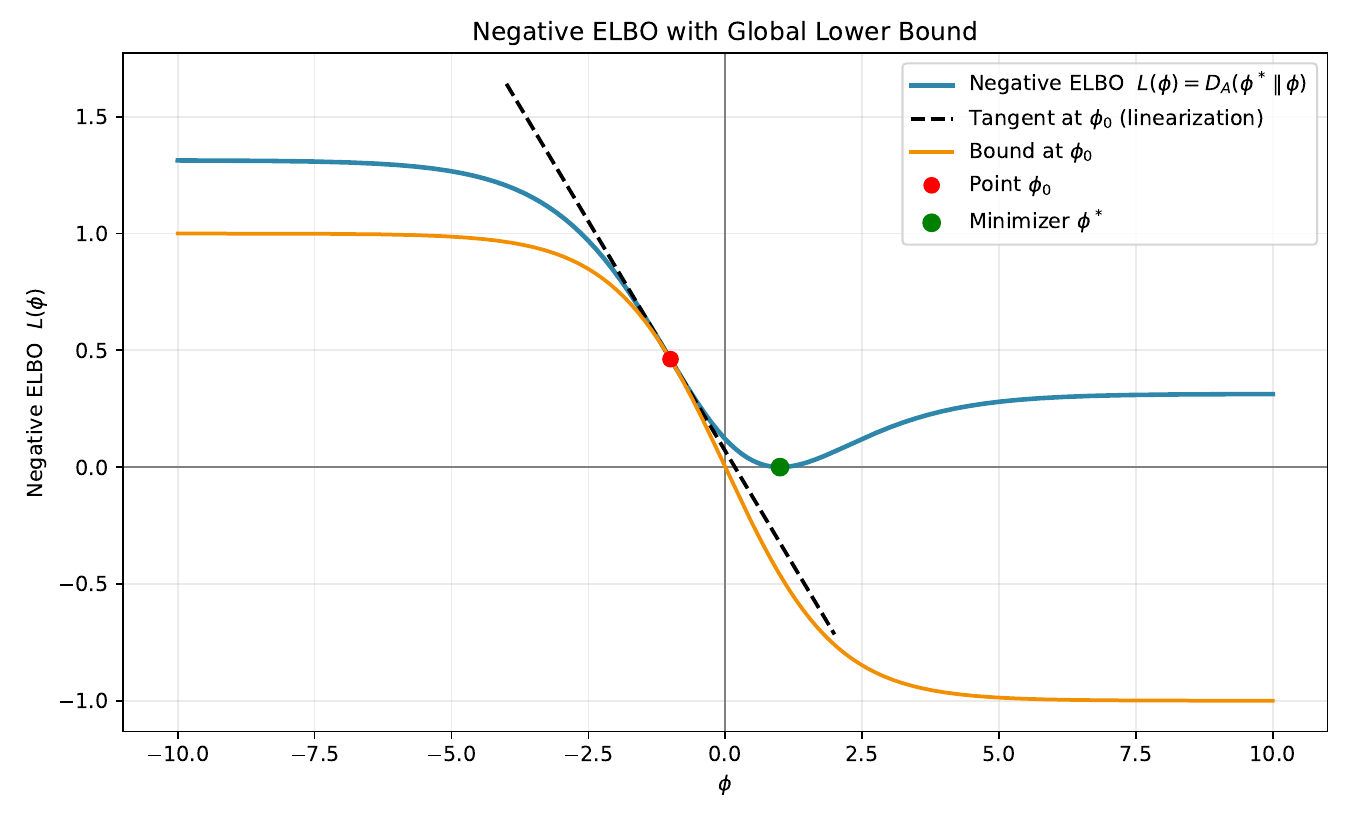}
  \caption{
  Geometry of the negative ELBO for the Bernoulli example.
  The blue curve is $L(\phi) = D_A(\phi^{\*}\!\parallel \phi)$ with the monotonicity bound from $\phi_0$,
  which provides a global inequality.
  }
  \label{fig:elbo-geometry}
\end{figure}

\section{Ray-wise Analysis and Two-sided Bounds} \label{sec: 5}

While the Bregman representation and monotonicity property lay the foundation for structural insights, we still need more information to yield concrete convergence rates. We note that these properties are global in nature, whereas optimization algorithms follow specific trajectories through parameter space. To bridge this gap, we develop a ray-wise analysis that examines the objective's behavior along line segments connecting any point to the optimum.

This approach is natural in the exponential family setting because the Fisher information matrix varies smoothly along parameter paths \citep{amari2016}, allowing us to bound the objective's curvature in terms of the spectral properties of $H(\phi)$ along specific geometric directions. Rather than requiring pessimistic global bounds \citep{xu2018, bottou2018}, this ray-wise perspective yields sharp local estimates that directly inform convergence analysis.

\subsection{Ray-wise Spectral Envelopes}

For any point $\phi \in \Omega$, we define the ray segment connecting the optimum to $\phi$ as:
\begin{equation}
\phi_s := \phi^* + s(\phi - \phi^*), \quad s \in [0,1]
\end{equation}

Along this segment, we bound the eigenvalues of the Fisher information matrix, which are defined in Definition \ref{def:spectral_bounds}.

\begin{definition}[Ray-wise Spectral Bounds]\label{def:spectral_bounds}
For $\phi \in \Omega$, define:
\begin{align}
\alpha(\phi) &:= \inf_{s \in [0,1]} \lambda_{\min}(H(\phi_s)) \\
\beta(\phi) &:= \sup_{s \in [0,1]} \lambda_{\max}(H(\phi_s))
\end{align}
\end{definition}

These quantities capture the best and worst conditioning of the Fisher information along the ray from the optimum to $\phi$. We find that these bounds are finite and continuous in nature. 

\begin{theorem}[Continuity and Finiteness]\label{thm:continuity}
Under Assumption 1, for each $\phi \in \Omega$:
\begin{enumerate}
\item The map $s \mapsto \lambda_{\max}(H(\phi_s))$ is continuous on $[0,1]$
\item $\beta(\phi) < \infty$
\item If the segment $[\phi^*, \phi]$ lies in a compact subset of $\Omega$, then $\alpha(\phi) > 0$
\end{enumerate}
\end{theorem}

\begin{proof}[Proof Sketch]
Continuity follows from the composition of continuous maps and Lipschitz continuity of the largest eigenvalue function. Finiteness follows from the extreme value theorem on $[0,1]$. See Appendix C.1 for details.
\end{proof}

\subsection{Integral Representation}

To lay the foundation of our raywise analysis, we use the integral representation of the objective function that expresses $L(\phi)$ as a weighted average of quadratic forms along the ray.

\begin{theorem}[Integral Representation]\label{thm:integral-rep}
Let $\delta := \phi - \phi^*$. Then:
\begin{equation}\label{eq:integral-rep}
L(\phi) = \int_0^1 s \cdot \delta^T H(\phi_s) \delta \, ds
\end{equation}
\end{theorem}

\begin{proof}[Proof Sketch]
Define $a(t) := A(\phi^* + t\delta)$ and apply the fundamental theorem of calculus twice to express $L(\phi)$ as a double integral, then change the order of integration. See Appendix C.2 for details. 
\end{proof}
This representation reveals that $L(\phi)$ is a weighted integral of the quadratic form $\delta^T H(\phi_s) \delta$ along the ray, with weight function $s$ that emphasizes points closer to $\phi$.

\subsection{Two-sided Quadratic Bounds}

The integral representation immediately yields precise two-sided bounds on the objective function.

\begin{theorem}[Ray-wise Quadratic Bounds]\label{thm:ray-bounds}
For any $\phi \in \Omega$:
\begin{equation}\label{eq:ray-bounds}
\frac{\alpha(\phi)}{2} \|\phi - \phi^*\|^2 \leq L(\phi) \leq \frac{\beta(\phi)}{2} \|\phi - \phi^*\|^2
\end{equation}
\end{theorem}

\begin{proof}[Proof Sketch]
Apply spectral bounds to the integrand in \eqref{eq:integral-rep} and integrate. The factor $1/2$ comes from $\int_0^1 s \, ds = 1/2$. See Appendix C.3 for details. 
\end{proof}
We visualize these two spectral bounds in Figure \ref{fig:ray-geometry} for the same setting as mentioned in the example in \ref{sec:monotone}.
\begin{figure}[tbh]
  \centering
  \includegraphics[width=1\linewidth]{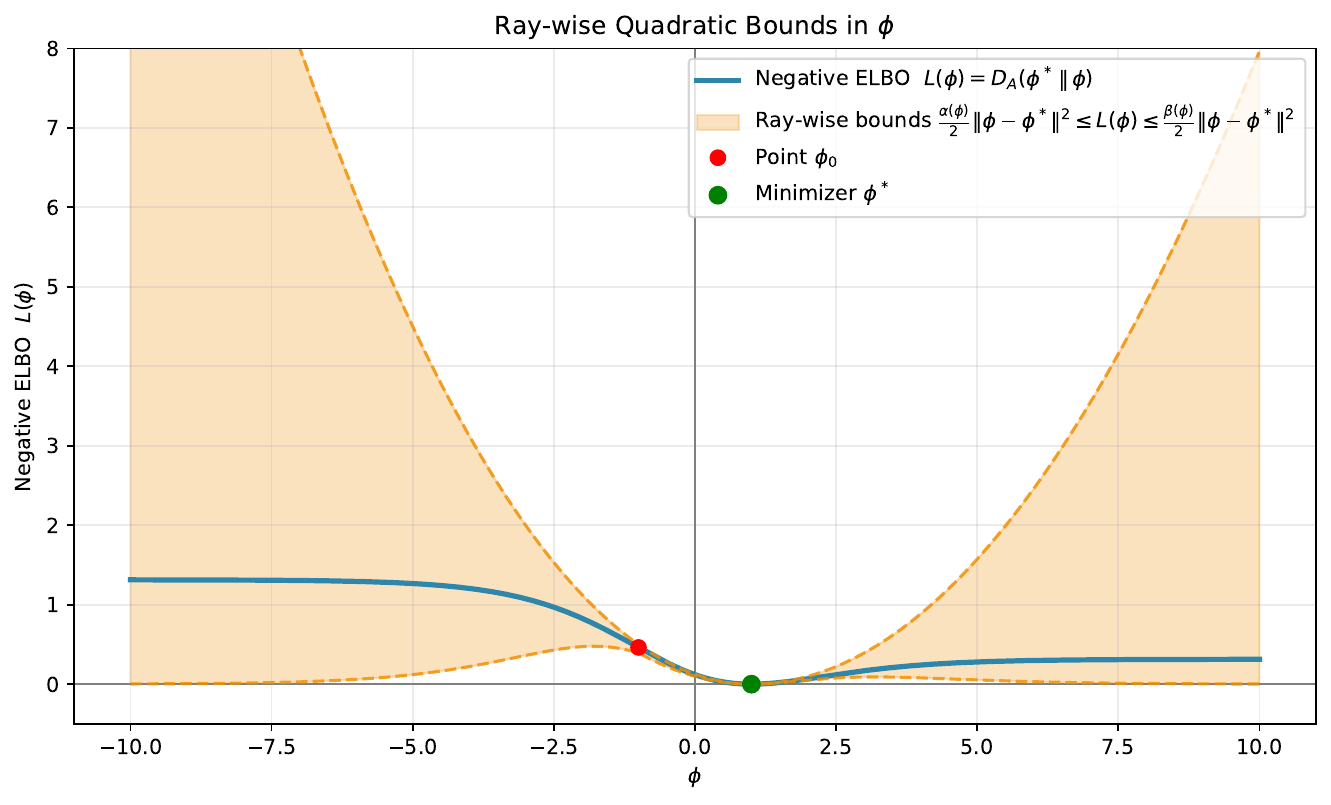}
  \caption{
  Ray-wise quadratic bounds for Bernoulli exponential family. The negative ELBO $L(\phi)$ (blue) lies within adaptive bounds $\frac{\alpha(\phi)}{2}\|\phi - \phi^*\|^2 \leq L(\phi) \leq \frac{\beta(\phi)}{2}\|\phi - \phi^*\|^2$ (orange envelope)
  }
  \label{fig:ray-geometry}
\end{figure}

The bounds show that $\alpha(\phi)$ and $\beta(\phi)$ vary across the parameter space. Near $\phi^*$ (the minimizer), the bounds are tight, but they expand significantly as we move away. These bounds help us to find following inequalities mentioned in Theorem \ref{thm:one-point}.

\begin{theorem}[One-point Inequalities]\label{thm:one-point}
Let $\delta := \phi - \phi^*$. Then:
\begin{align}
\alpha(\phi) \|\delta\|^2 &\leq \langle \nabla L(\phi), \delta \rangle \leq \beta(\phi) \|\delta\|^2 \label{eq:grad-bounds}\\
\frac{2\alpha(\phi)}{\beta(\phi)} L(\phi) &\leq \langle \nabla L(\phi), \delta \rangle \leq \frac{2\beta(\phi)}{\alpha(\phi)} L(\phi) \label{eq:PL-bounds}
\end{align}
provided $\alpha(\phi) > 0$.
\end{theorem}

\begin{proof}[Proof sketch]
Let $\delta:=\phi-\phi^*$. From the gradient identity $\nabla L(\phi)=H(\phi)\delta$ and the ray-wise envelopes
$\alpha(\phi):=\inf_{s\in[0,1]}\lambda_{\min}(H(\phi_s))$,
$\beta(\phi):=\sup_{s\in[0,1]}\lambda_{\max}(H(\phi_s))$,
the inclusion $s=1$ implies
$\lambda_{\min}(H(\phi))\!\ge\!\alpha(\phi)$ and
$\lambda_{\max}(H(\phi))\!\le\!\beta(\phi)$.  
The Rayleigh quotient then gives
$\alpha(\phi)\|\delta\|^2 \le \delta^\top H(\phi)\delta 
= \langle \nabla L(\phi),\delta\rangle \le \beta(\phi)\|\delta\|^2$.

For the PL-type bounds \eqref{eq:PL-bounds}, combine the integral representation
$L(\phi)=\int_0^1 s\,\delta^\top H(\phi_s)\delta\,ds$
with the spectral envelopes along the ray to obtain
$\tfrac{\alpha(\phi)}{2}\|\delta\|^2 \le L(\phi) \le \tfrac{\beta(\phi)}{2}\|\delta\|^2$.
Eliminating $\|\delta\|^2$ using \eqref{eq:grad-bounds} yields
$\tfrac{2\alpha(\phi)}{\beta(\phi)}L(\phi) \le \langle \nabla L(\phi),\delta\rangle 
\le \tfrac{2\beta(\phi)}{\alpha(\phi)}L(\phi)$, with the left inequality requiring $\alpha(\phi)>0$. See Appendix C.4 for details. 
\end{proof}

The second inequality \eqref{eq:PL-bounds} is a Polyak-Łojasiewicz type condition  \citep{polyak1963, karimi2016} with ray-dependent constants. Unlike classical PL inequalities that require global bounds, this version adapts to local geometry and can be much more favorable than worst-case global analysis.

This ray-wise analysis framework allows us to establish convergence rates that adapt to local problem structure rather than relying on global worst-case bounds. They are helpful especially because the optimization algorithms follow specific trajectories where favorable local geometry can be exploited, even when global properties are not favorable.

\section{Convergence of Natural and Euclidean Gradient Methods}
\label{sec:convergence}
We now bring together the geometric ingredients from Section \ref{sec: 5} to analyze the convergence
of two first-order optimization methods for minimizing the negative ELBO:
natural gradient descent (NGD) and standard Euclidean gradient descent (GD).
Throughout we work under Assumptions~1--2 and the ray-wise spectral envelopes
$\alpha(\phi),\beta(\phi)$ introduced earlier.
Recall the key identities
\begin{equation}
   L(\phi) = D_A(\phi^*\|\phi) 
\end{equation}
\begin{align}\label{eq:grad-identities}
\nabla L(\phi) &= H(\phi)(\phi-\phi^*), \\
H(\phi)^{-1}\nabla L(\phi) &= \phi-\phi^* .
\end{align}
We write $\delta_k := \phi_k - \phi^*$ for the error vector.

\subsection{Natural Gradient Descent (NGD)} \label{sec:ngd}

The natural gradient method leverages the Fisher information matrix to precondition gradient steps, adapting to the natural Riemannian geometry of the exponential family \citep{amari1998, amari2016}. This geometric perspective reveals why natural gradients are particularly effective for variational inference problems.

The NGD update is
\begin{equation}
\phi_{k+1} = \phi_k - \eta_k H(\phi_k)^{-1}\nabla L(\phi_k).
\end{equation}
By \eqref{eq:grad-identities}, this reduces to the simple recursion
\begin{align}\label{eq:ngd-update}
\phi_{k+1} &= \phi^* + (1-\eta_k)(\phi_k-\phi^*), \\
\delta_{k+1} &= (1-\eta_k)\delta_k .
\end{align}

The Bregman structure reveals that all iterates lie on the line segment connecting the optimum and initialization.

\begin{lemma}[Ray invariance]
\label{lem:ray-invariance}
All NGD iterates lie on the line segment connecting $\phi^*$ and the initialization $\phi_0$.
Moreover, the distance recursion is exact:
\begin{equation}
\|\delta_{k+1}\| = |1-\eta_k|\,\|\delta_k\|.
\end{equation}
\end{lemma}

This ray invariance property enables us to establish convergence rates that are independent of the conditioning of the Fisher information matrix.

\begin{theorem}[Distance contraction for NGD] \label{th:theorem7}
\label{thm:ngd-contraction}
If $0<\eta<2$ is constant, then
\begin{equation}
\|\phi_k-\phi^*\| = |1-\eta|^k \|\phi_0-\phi^*\|.
\end{equation}
In particular, the choice $\eta=1$ yields convergence to the optimum in a single step. It occurs due to setting of our assumptions. 
\end{theorem}

We show below that nature of the Bregman divergence translates the linear parameter convergence into quadratic function value decay.

\begin{corollary}[Function-value decay for $0<\eta\leq 1$]
\label{cor:ngd-func}
Let $\kappa_0 := \beta(\phi_0)/\alpha(\phi_0)$. Then for constant $0<\eta\leq 1$,
\begin{equation}
L(\phi_k) \leq \kappa_0 \, |1-\eta|^{2k}\, L(\phi_0).
\end{equation}
\end{corollary}

For adaptive step size schedules, the convergence analysis extends naturally through the cumulative step size framework.

\begin{corollary}[Diminishing steps]
\label{cor:ngd-diminishing}
If $0<\eta_k\leq 1$ for all $k$, then
\begin{align}
\|\phi_k-\phi^*\| &\leq \exp\left(-\sum_{i=0}^{k-1}\eta_i\right)\|\phi_0-\phi^*\|, \\
L(\phi_k) &\leq \kappa_0 \exp\left(-2\sum_{i=0}^{k-1}\eta_i\right)L(\phi_0).
\end{align}
For example, with $\eta_i=c/i$ and $0<c<1$, one obtains
$\|\phi_k-\phi^*\|=O(k^{-c})$ and $L(\phi_k)=O(k^{-2c})$.
\end{corollary}

NGD leverages the Bregman geometry exactly: the natural gradient points directly
from $\phi$ to $\phi^*$, producing a straight-line contraction in parameter space.
The rate $|1-\eta|$ is independent of conditioning which explains this the method's empirical
efficiency. The decay of the function value follows quadratically as $|1-\eta|^2$ when the steps are not overshot.

\subsection{Euclidean Gradient Descent (GD)}

Standard gradient descent operates in the Euclidean parameter space, ignoring the natural Riemannian structure of the exponential family. This leads to convergence behavior that depends critically on the conditioning of the Fisher information matrix, contrasting sharply with the condition-independent performance of NGD.

The Euclidean update is
\begin{align}
\phi_{k+1} &= \phi_k - \gamma_k \nabla L(\phi_k), \\
\delta_{k+1} &= (I - \gamma_k H(\phi_k))\delta_k .
\end{align}

The convergence analysis for GD requires careful treatment of the eigenvalue spectrum of $H(\phi_k)$, as the method behaves like descent on a locally quadratic approximation whose conditioning determines the convergence rate.

\begin{theorem}[Per-iteration contraction for GD]
\label{thm:gd-contraction}
Fix $k$ and choose $\gamma_k \in (0,2/\beta(\phi_k))$.
Then
\begin{align}
\|\delta_{k+1}\| &\leq \rho_k \|\delta_k\|, \\
\rho_k &:= \max\{|1-\gamma_k \alpha(\phi_k)|,\;|1-\gamma_k \beta(\phi_k)|\}<1.
\end{align}
The optimal step is $\gamma_k^* = 2/(\alpha(\phi_k)+\beta(\phi_k))$,
yielding
\begin{equation}
\rho_k^* = \frac{\beta(\phi_k)-\alpha(\phi_k)}{\beta(\phi_k)+\alpha(\phi_k)}.
\end{equation}
\end{theorem}

The function value analysis follows from the parameter convergence, though the relationship is more complex than in the NGD case due to the changing geometry along the optimization path.

\begin{corollary}[One-step function-value decrease]
\label{cor:gd-func}
For any $\gamma_k \in (0,2/\beta(\phi_k))$,
\begin{equation}
L(\phi_{k+1}) \leq
\frac{\beta(\phi_{k+1})}{\alpha(\phi_k)}\,
\rho_k^2\,L(\phi_k).
\end{equation}
\end{corollary}

Unlike NGD, GD behaves like descent on a locally quadratic model whose eigenvalues
lie between $\alpha(\phi_k)$ and $\beta(\phi_k)$. Its best contraction factor is  $(\beta-\alpha)/(\beta+\alpha)$, which deteriorates as conditioning
worsens. This contrast explains the superior empirical performance of NGD:
it contracts independently of $\alpha,\beta$, whereas GD is bottlenecked by 
condition number $\kappa(\phi)=\beta(\phi)/\alpha(\phi)$.

\section{Numerical Experiments}
We present two sets of experiments: first, a trajectory comparison of NGD and GD on a Bernoulli exponential family, and second, a convergence study on a high-dimensional Gaussian VI problem.

\subsection{Trajectory comparison} \label{sec:track}
We compare the trajectory taken by natural gradient descent and Euclidean gradient descent on a 2-dimensional Bernoulli exponential family.

\begin{figure}[tbh]
  \centering  \includegraphics[width=1\linewidth]{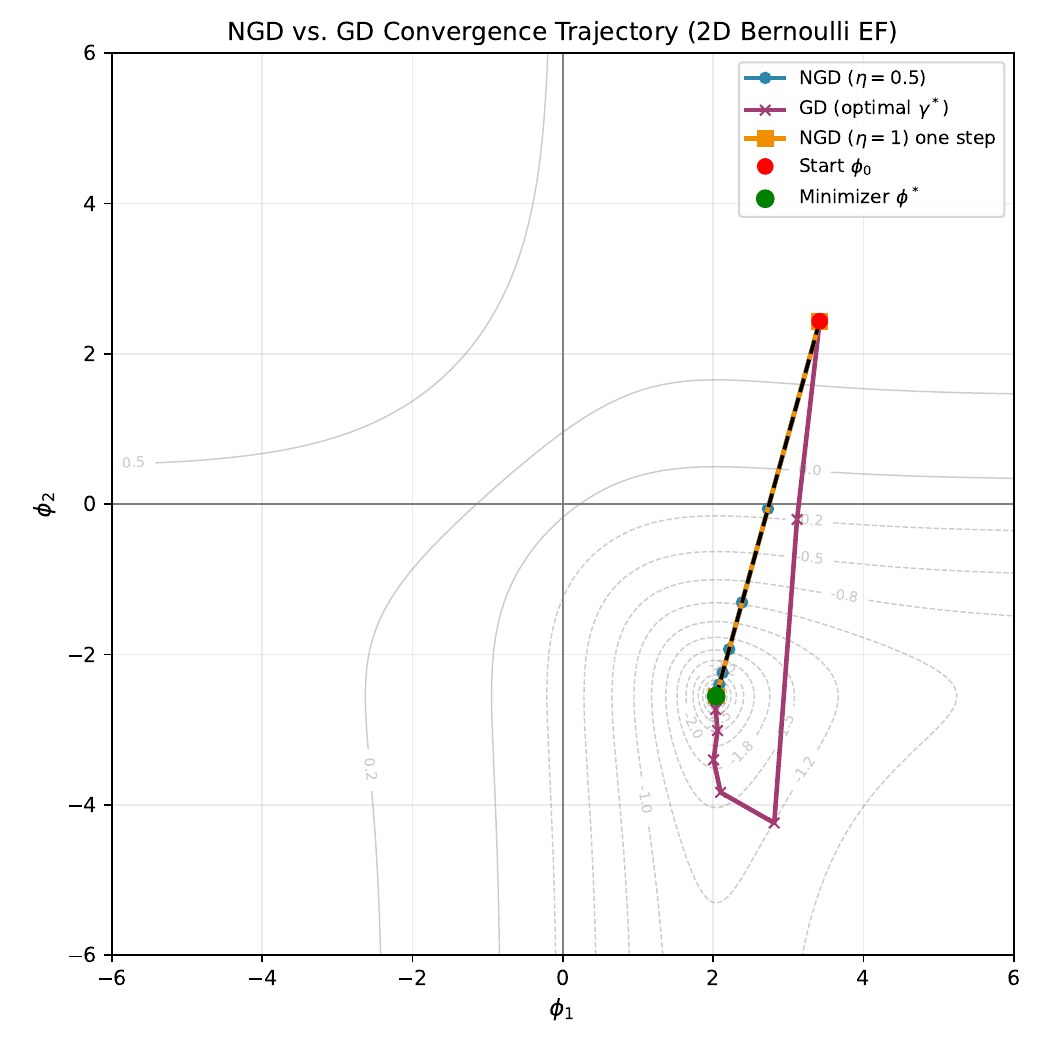}
  \caption{Trajectory comparison of Natural Gradient Descent (NGD) Vs. Gradient Descent
  }
  \label{fig:ray-invariance}
\end{figure}

We show three distinct optimization trajectories overlaid on the contour plot of the negative ELBO landscape. Starting from the red point $\phi_0$, we observe that the blue trajectory (NGD with $\eta=0.5$) descends directly along the black dashed line toward the green minimizer $\phi^*$, with each step moving exactly halfway along the remaining distance. The orange trajectory (NGD with $\eta=1$) demonstrates a single jump directly from $\phi_0$ to $\phi^*$, traversing the entire distance in one step. On the other hand, the purple trajectory (GD with optimal $\gamma^*$) begins in the same direction but gradually curves away from the direct path, sweeping through the lower portion of the parameter space before eventually approaching $\phi^*$ from below. The gray contour lines reveal that while NGD trajectories cut straight across these level sets along the steepest descent direction in the natural metric, the GD trajectory follows the geometry imposed by the Euclidean metric, which does not align with the underlying Bregman structure of the objective function.

\subsection{GD vs.\ NGD on Gaussian VI (d=20)} \label{sec:vae}
We take both the true posterior and  approximator to be Gaussian with fixed covariance: \\
\begin{align*}
p(z \mid x) &= \mathcal{N}(z \mid \phi^*, I), \\
q_\phi(z) &= \mathcal{N}(z \mid \phi, I).
\end{align*}

so that the Fisher information is diagonal. 
In this setting, the negative ELBO reduces to the quadratic form which we minimize using GD and NGD. See Appendix B.2 for experimental setup.  \\

We compare Euclidean Gradient Descent (GD) and Natural Gradient Descent (NGD) on a $d=20$ Gaussian variational inference problem, where the spectrum of the Fisher information matrix is controlled through different ratios $\alpha/\beta$. 

\begin{figure}[tbh]
  \centering  \includegraphics[width=1\linewidth]{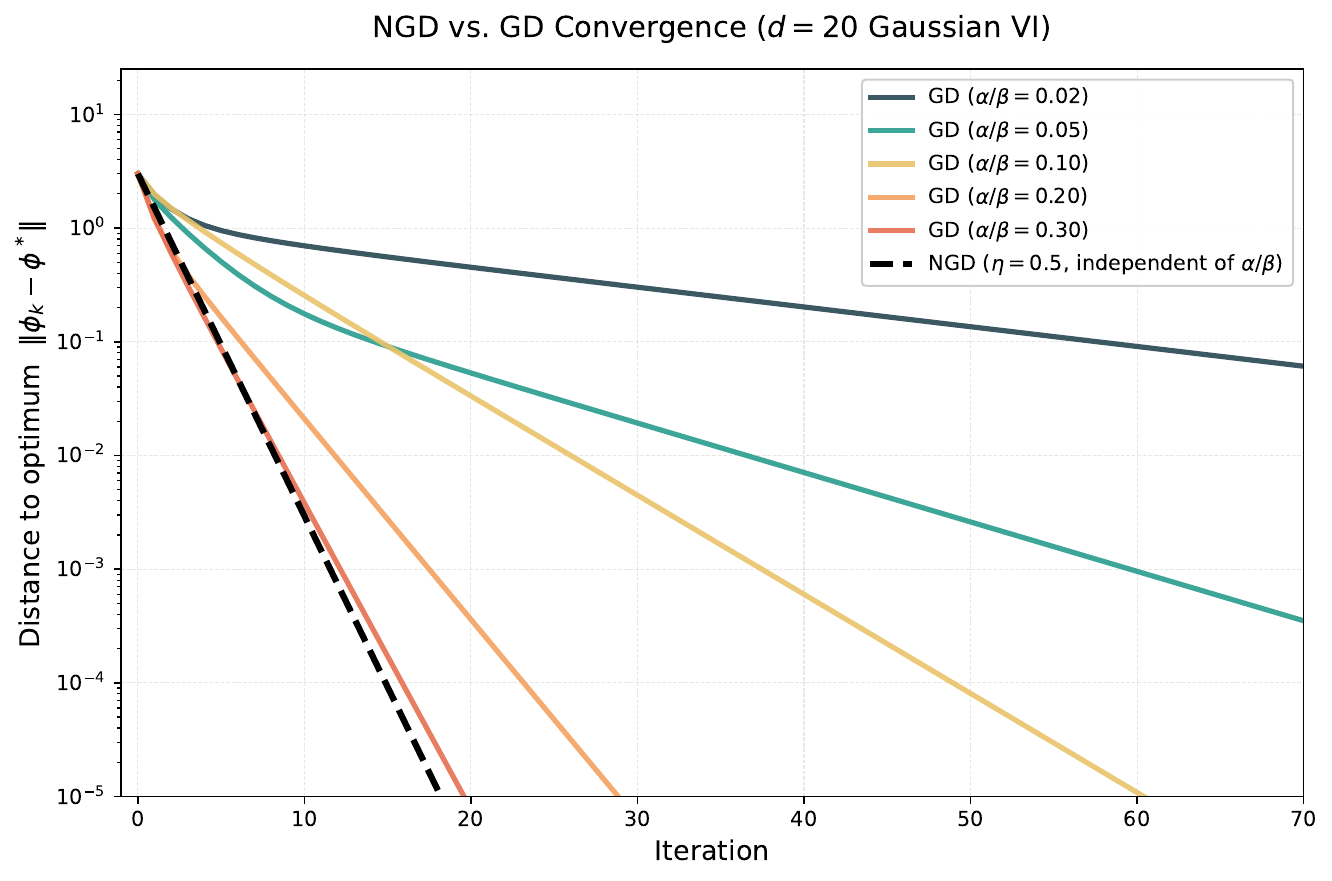}
  \caption{Trajectory comparison of Natural Gradient Descent (NGD) Vs. Gradient Descent
  }
  \label{fig:vae}
\end{figure}

From Definition~\ref{def:spectral_bounds}, $\alpha(\phi)$ and $\beta(\phi)$ bound the smallest and largest eigenvalues of the Fisher information matrix along the ray to the optimum, and their ratio $\alpha/\beta$ quantifies how well-conditioned the problem is. Figure~\ref{fig:vae} shows that GD trajectories slow down significantly as $\alpha/\beta$ decreases, reflecting the dependence of the contraction factor on the eigenvalue spread given in Theorem~\ref{thm:gd-contraction}. In contrast, NGD converges linearly at a uniform rate independent of $\alpha/\beta$, moving directly along the ray to the optimum as guaranteed by Lemma~\ref{lem:ray-invariance} and Theorem~\ref{thm:ngd-contraction}. These results empirically confirm our ray-wise analysis from Section~\ref{sec:convergence}: GD is bottlenecked by poor conditioning, while NGD exploits the Bregman geometry to achieve robust convergence across all regimes.

\section{Conclusion}
We developed a geometric framework and analyzed the convergence of variational inference by uncovering properties--a ray-wise monotonicity and quadratic bounding behavior which were validated through numerical experiments. One limitation of our work is that the analysis hinges on the posterior being a member of the exponential family; however, we believe that the result can be extended with additional work.   

\section{Acknowledgment}
This work was carried out at the Information Theory, Probability, and Statistics Unit, Okinawa Institute of Science and Technology Graduate University (OIST), as part of a research internship.

\bibliography{references}

\clearpage
\appendix
\thispagestyle{empty}

\onecolumn
\appendix

\section{KL Divergence and the ELBO}
\label{app:kl-elbo}

The Kullback--Leibler divergence is
\[
D_{\mathrm{KL}}(q_\phi(z)\,\|\,p(z \mid x)) 
= \mathbb{E}_{q_\phi}\!\left[ \log \frac{q_\phi(z)}{p(z \mid x)} \right].
\]
Using $p(z \mid x) = \tfrac{p(x,z)}{p(x)}$ yields
\begin{align*}
D_{\mathrm{KL}}(q_\phi(z)\,\|\,p(z \mid x))
&= \mathbb{E}_{q_\phi}\!\left[ \log q_\phi(z) - \log p(x,z) + \log p(x) \right] \\
&= \underbrace{\mathbb{E}_{q_\phi}\!\left[ \log q_\phi(z) - \log p(x,z) \right]}_{L(\phi)} \;+\; \log p(x).
\end{align*}
Here, $L(\phi)$ is the negative ELBO given in Equation~\eqref{eq:neg_elbo}. Since $\log p(x)$ is independent of $\phi$, minimizing the KL divergence is equivalent to minimizing $L(\phi)$, or equivalently maximizing the ELBO:
\[
\arg\min_\phi D_{\mathrm{KL}}(q_\phi \,\|\, p(z\mid x)) 
\;=\; \arg\min_\phi L(\phi) 
\;=\; \arg\max_\phi \mathrm{ELBO}(\phi).
\]

\section{Experimental Details}
\label{app:exp-details}

\subsection{Details for Section~\ref{sec:track}}
In the trajectory comparison experiment with the Bernoulli exponential family (Figure~\ref{fig:ray-invariance}), we visualize optimization paths of GD and NGD on the negative ELBO landscape. The setup is as follows:
\begin{enumerate}
    \item We initialize the parameter $\phi_0$ at a point away from the minimizer $\phi^*$ and generate contour plots of the negative ELBO.
    \item NGD is run with step sizes $\eta=0.5$ and $\eta=1.0$, while GD uses the step size $\gamma^* = 2/(\alpha(\phi)+\beta(\phi))$ as given in Theorem~\ref{thm:gd-contraction}.
\end{enumerate}

\subsection{Details for Section~\ref{sec:vae}}
For the Gaussian variational inference experiment in $d=20$ dimensions (Figure~\ref{fig:vae}), we vary the conditioning of the Fisher information matrix by fixing $\beta=1$ and selecting different values of $\alpha$, so that $\alpha/\beta$ controls the conditioning. The procedure is:
\begin{enumerate}
    \item We consider ratios $\alpha/\beta \in \{0.02,0.05,0.10,0.20,0.30\}$.
    \item GD is run with $\gamma^* = 2/(\alpha+\beta)$, and NGD with constant step size $\eta=0.5$.
\end{enumerate}

\section{Detailed Proofs for Section 5}

\subsection{Proof of Theorem \ref{thm:continuity}}

\begin{proof}
(1) The map $s \mapsto \phi_s$ is affine, hence continuous. Since $H$ is continuous on $\Omega$ by Assumption 1 and the largest eigenvalue function $M \mapsto \lambda_{\max}(M)$ is Lipschitz continuous with respect to the operator norm on symmetric matrices, their composition is continuous.

Specifically, for symmetric matrices $M, N$, we have:
$$|\lambda_{\max}(M) - \lambda_{\max}(N)| \leq \|M - N\|_2$$

Therefore, $s \mapsto \lambda_{\max}(H(\phi_s))$ is continuous on $[0,1]$.

(2) By the extreme value theorem, any continuous function on a compact set attains its maximum. Since $[0,1]$ is compact and $s \mapsto \lambda_{\max}(H(\phi_s))$ is continuous, $\beta(\phi) < \infty$.

(3) If $[\phi^*, \phi] \subset K$ for some compact $K \subset \Omega$, then by continuity of $H$ and positive definiteness, there exists $m > 0$ such that $\lambda_{\min}(H(\psi)) \geq m$ for all $\psi \in K$. Thus $\alpha(\phi) \geq m > 0$.
\end{proof}

\subsection{Proof of Theorem \ref{thm:integral-rep}}

\begin{proof}
Define $a(t) := A(\phi^* + t\delta)$ for $t \in [0,1]$. Then:
\begin{align}
a'(t) &= \langle \nabla A(\phi_t), \delta \rangle = \langle \mu(\phi_t), \delta \rangle \\
a''(t) &= \delta^T H(\phi_t) \delta
\end{align}

From the Bregman representation,
$$L(\phi) = A(\phi^*) - A(\phi) + \langle \mu(\phi), \delta \rangle$$

By the fundamental theorem of calculus:
$$A(\phi) - A(\phi^*) = a(1) - a(0) = \int_0^1 a'(t) dt = \int_0^1 \langle \mu(\phi_t), \delta \rangle dt$$

Similarly:
$$\mu(\phi) - \mu(\phi_t) = \int_t^1 H(\phi_s) \delta \, ds$$

Substituting:
\begin{align}
L(\phi) &= -\int_0^1 \langle \mu(\phi_t), \delta \rangle dt + \langle \mu(\phi), \delta \rangle \\
&= \int_0^1 \langle \mu(\phi) - \mu(\phi_t), \delta \rangle dt \\
&= \int_0^1 \left\langle \int_t^1 H(\phi_s) \delta \, ds, \delta \right\rangle dt \\
&= \int_0^1 \int_t^1 \delta^T H(\phi_s) \delta \, ds \, dt
\end{align}

Changing the order of integration over the triangular region $\{(t,s) : 0 \leq t \leq s \leq 1\}$:
$$L(\phi) = \int_0^1 \int_0^s \delta^T H(\phi_s) \delta \, dt \, ds = \int_0^1 s \cdot \delta^T H(\phi_s) \delta \, ds$$
\end{proof}

\subsection{Proof of Theorem \ref{thm:ray-bounds}}

\begin{proof}
From the integral representation and spectral bounds:
$$\lambda_{\min}(H(\phi_s)) \|\delta\|^2 \leq \delta^T H(\phi_s) \delta \leq \lambda_{\max}(H(\phi_s)) \|\delta\|^2$$

By definition of $\alpha(\phi)$ and $\beta(\phi)$:
$$\alpha(\phi) \|\delta\|^2 \leq \delta^T H(\phi_s) \delta \leq \beta(\phi) \|\delta\|^2$$

for all $s \in [0,1]$. Integrating with weight $s$:
$$\alpha(\phi) \|\delta\|^2 \int_0^1 s \, ds \leq L(\phi) \leq \beta(\phi) \|\delta\|^2 \int_0^1 s \, ds$$

Since $\int_0^1 s \, ds = 1/2$, we obtain the desired bounds.
\end{proof}

\subsection{Proof of Theorem~\ref{thm:one-point}}

\begin{proof}
We have $\delta:=\phi-\phi^*$, $\phi_s:=\phi^*+s(\phi-\phi^*)$, and
\[
\alpha(\phi):=\inf_{s\in[0,1]}\lambda_{\min}(H(\phi_s)),
\qquad
\beta(\phi):=\sup_{s\in[0,1]}\lambda_{\max}(H(\phi_s)).
\]
We use the identities (proved earlier):
\[
\nabla L(\phi)=H(\phi)\delta,
\qquad
L(\phi)=\int_0^1 s\,\delta^\top H(\phi_s)\delta\,ds.
\]
(i)
Because $s=1$ lies in $[0,1]$, the definitions of $\alpha(\phi)$ and $\beta(\phi)$ imply
\[
\lambda_{\min}(H(\phi)) \;\ge\; \alpha(\phi),
\qquad
\lambda_{\max}(H(\phi)) \;\le\; \beta(\phi).
\]
Hence, by the Rayleigh quotient for symmetric positive definite matrices,
\[
\alpha(\phi)\,\|\delta\|^2
\;\le\; \delta^\top H(\phi)\,\delta
\;\le\; \beta(\phi)\,\|\delta\|^2.
\]
Using $\langle \nabla L(\phi),\delta\rangle=\delta^\top H(\phi)\delta$ gives
\[
\alpha(\phi)\,\|\delta\|^2 \;\le\; \langle \nabla L(\phi),\delta\rangle
\;\le\; \beta(\phi)\,\|\delta\|^2,
\]
establishing \eqref{eq:grad-bounds}.

(ii)
From the integral representation and the spectral envelopes along the ray,
\[
\alpha(\phi)\,\|\delta\|^2
\;\le\; \delta^\top H(\phi_s)\delta
\;\le\; \beta(\phi)\,\|\delta\|^2
\quad\text{for all } s\in[0,1].
\]
Integrating against $s$ yields the ray-wise quadratic bounds (Theorem~\ref{thm:ray-bounds}):
\[
\frac{\alpha(\phi)}{2}\,\|\delta\|^2 \;\le\; L(\phi) \;\le\; \frac{\beta(\phi)}{2}\,\|\delta\|^2.
\]
Combine these with \eqref{eq:grad-bounds}.

\emph{Lower PL bound:} 
\[
\langle \nabla L(\phi),\delta\rangle
\;\ge\; \alpha(\phi)\,\|\delta\|^2
\;\ge\; \alpha(\phi)\cdot\frac{2}{\beta(\phi)}\,L(\phi)
\;=\; \frac{2\alpha(\phi)}{\beta(\phi)}\,L(\phi).
\]

\emph{Upper PL bound:}
\[
\langle \nabla L(\phi),\delta\rangle
\;\le\; \beta(\phi)\,\|\delta\|^2
\;\le\; \beta(\phi)\cdot\frac{2}{\alpha(\phi)}\,L(\phi)
\;=\; \frac{2\beta(\phi)}{\alpha(\phi)}\,L(\phi),
\]

\end{proof}

\section{Detailed Proofs for Section~\ref{sec:convergence}}

\subsection{Proof of Lemma~\ref{lem:ray-invariance}}
\begin{proof}
The NGD update is
\[
\phi_{k+1} = \phi_k - \eta_k H(\phi_k)^{-1}\nabla L(\phi_k).
\]
From the identity $H(\phi)^{-1}\nabla L(\phi)=\phi-\phi^*$,
\[
\phi_{k+1} = \phi_k - \eta_k(\phi_k-\phi^*).
\]
Rearrange:
\[
\phi_{k+1} = \phi^* + (1-\eta_k)(\phi_k-\phi^*).
\]
Subtracting $\phi^*$ gives $\delta_{k+1}=(1-\eta_k)\delta_k$.
Thus $\delta_{k+1}$ is always a scalar multiple of $\delta_k$,
so every iterate remains on the line through $\phi^*$ and $\phi_0$.
Finally, taking norms:
\[
\|\delta_{k+1}\| = |1-\eta_k|\|\delta_k\|.
\]
\end{proof}
\subsection{Proof of Theorem~\ref{thm:ngd-contraction}}
\begin{proof}
With constant step $\eta$, the recursion is
$\delta_{k+1}=(1-\eta)\delta_k$. Iterating $k$ times:
\[
\delta_k = (1-\eta)^k \delta_0.
\]
Hence $\|\delta_k\| = |1-\eta|^k \|\delta_0\|$.
If $\eta=1$, then $\delta_1=0$, so $\phi_1=\phi^*$ after one step.
\end{proof}

\subsection{Proof of Corollary~\ref{cor:ngd-func}}
\begin{proof}
From the quadratic bounds of Section~5:
\[
\frac{\alpha(\phi)}{2}\|\delta\|^2 \leq L(\phi) \leq \frac{\beta(\phi)}{2}\|\delta\|^2.
\]
Applying this at $k$ and $k+1$, and using
$\|\delta_{k+1}\|=|1-\eta|\|\delta_k\|$, we obtain
\[
\frac{L(\phi_{k+1})}{L(\phi_k)}
\leq \frac{\tfrac{1}{2}\beta(\phi_{k+1})\|\delta_{k+1}\|^2}{\tfrac{1}{2}\alpha(\phi_k)\|\delta_k\|^2}
= \frac{\beta(\phi_{k+1})}{\alpha(\phi_k)}|1-\eta|^2.
\]
When $0<\eta\leq 1$, the nesting property of Section~5 implies
$\beta(\phi_{k+1})\leq \beta(\phi_0)$ and $\alpha(\phi_k)\geq \alpha(\phi_0)$.
Thus
\[
L(\phi_{k+1}) \leq \kappa_0 |1-\eta|^2 L(\phi_k),\qquad
\kappa_0=\beta(\phi_0)/\alpha(\phi_0).
\]
Iterating this inequality yields the claimed bound.
\end{proof}

\subsection{Proof of Corollary~\ref{cor:ngd-diminishing}}
\begin{proof}
The recursion $\delta_{k+1}=(1-\eta_k)\delta_k$ gives
\[
\|\delta_k\| = \Big(\prod_{i=0}^{k-1}|1-\eta_i|\Big)\|\delta_0\|.
\]
Using $\log(1-x)\leq -x$ for $x\in(0,1]$,
\[
\prod_{i=0}^{k-1}(1-\eta_i) \leq \exp\!\Big(-\sum_{i=0}^{k-1}\eta_i\Big).
\]
Thus
\[
\|\delta_k\| \leq \exp\!\Big(-\sum_{i=0}^{k-1}\eta_i\Big)\|\delta_0\|.
\]
The function-value bound follows by combining with the quadratic bounds
as in the previous corollary.
For $\eta_i=c/i$, note $\sum_{i=1}^{k} c/i = c\log k + O(1)$,
so $\|\delta_k\|=O(k^{-c})$ and $L(\phi_k)=O(k^{-2c})$.

\end{proof}
\subsection{Proof of Theorem~\ref{thm:gd-contraction}}
\begin{proof}
Let $H(\phi_k)=Q\Lambda Q^\top$ with $\Lambda=\mathrm{diag}(\lambda_1,\dots,\lambda_d)$
and $Q$ orthogonal. Then
\[
\delta_{k+1} = (I-\gamma_k H(\phi_k))\delta_k
= Q(I-\gamma_k\Lambda)Q^\top \delta_k.
\]
Taking norms:
\[
\|\delta_{k+1}\| \leq \|I-\gamma_k\Lambda\|_2 \|\delta_k\|
= \max_i |1-\gamma_k\lambda_i| \,\|\delta_k\|.
\]
Since $\lambda_i\in[\alpha(\phi_k),\beta(\phi_k)]$, the worst-case factor is
\[
\rho_k(\gamma_k)=\max\{|1-\gamma_k\alpha(\phi_k)|,\;|1-\gamma_k\beta(\phi_k)|\}.
\]
If $0<\gamma_k<2/\beta(\phi_k)$ then $\rho_k<1$. Minimizing over $\gamma_k$
gives $\gamma_k^*=2/(\alpha+\beta)$ and $\rho_k^*=(\beta-\alpha)/(\beta+\alpha)$.

\end{proof}
\subsection{Proof of Corollary~\ref{cor:gd-func}}
\begin{proof}
Apply the quadratic bounds at $k$ and $k+1$:
\[
L(\phi_{k+1}) \leq \tfrac{1}{2}\beta(\phi_{k+1})\|\delta_{k+1}\|^2,\qquad
L(\phi_k) \geq \tfrac{1}{2}\alpha(\phi_k)\|\delta_k\|^2.
\]
Using $\|\delta_{k+1}\|\leq \rho_k\|\delta_k\|$ from Theorem~\ref{thm:gd-contraction},
\[
\frac{L(\phi_{k+1})}{L(\phi_k)} \leq \frac{\beta(\phi_{k+1})}{\alpha(\phi_k)}\rho_k^2.
\]
This establishes the result.
\end{proof}


\appendix

\end{document}